\documentclass[lettersize,journal]{IEEEtran}
\usepackage{amsmath,amsfonts}
\usepackage{amssymb,amsthm}
\usepackage{algorithmic}
\usepackage{algorithm}
\usepackage{array}
\usepackage[caption=false,font=normalsize,labelfont=sf,textfont=sf]{subfig}
\usepackage{textcomp}
\usepackage{stfloats}
\usepackage{url}
\usepackage{graphicx}
\usepackage{cite}
\usepackage{booktabs}
\usepackage{multirow}
\usepackage{xcolor}
\usepackage{colortbl}
\usepackage{hyperref}
\newtheorem{theorem}{Theorem}
\newtheorem{proposition}{Proposition}

\title{Rectified Schr\"odinger Bridge Matching for\\Few-Step Visual Navigation}

\author{Wuyang~Luan,~Junhui~Li,~Weiguang~Zhao,~Wenjian~Zhang,~Tieru~Wu$^{\dagger}$,~and~Rui~Ma$^{\dagger}$%
\thanks{$^{\dagger}$Corresponding authors (e-mail: wutr@jlu.edu.cn; ruim@jlu.edu.cn).}%
\thanks{W.~Luan, T.~Wu, and R.~Ma are with the School of Artificial Intelligence, Jilin University, Changchun 130012, China (e-mail: luanwy25@mails.jlu.edu.cn).}%
\thanks{J.~Li is with the College of Computer Science, Chongqing University, Chongqing 400044, China (e-mail: junhuili@stu.cqu.edu.cn).}%
\thanks{W.~Zhao is with the Department of Computer Science, University of Liverpool, Liverpool L69 3BX, UK (e-mail: weiguang.zhao@liverpool.ac.uk).}%
\thanks{W.~Zhang is with Changchun GenY Technology Co., Ltd., Changchun, China (e-mail: zhangwenjian@genycc.cn).}%
}

\begin{document}
\maketitle

\begin{abstract}
Visual navigation from monocular image streams requires an agent to transform high-dimensional visual observations into continuous, long-horizon action trajectories. Generative policies based on diffusion models and Schr\"odinger Bridges (SB) can model multimodal action distributions faithfully, yet they require dozens of integration steps due to high-variance stochastic transport. We propose Rectified Schr\"odinger Bridge Matching (RSBM), a visual-conditioned trajectory generation framework that exploits a shared velocity-field structure between standard Schr\"odinger Bridges ($\varepsilon\!=\!1$) and deterministic Optimal Transport ($\varepsilon\!\to\!0$), governed by a single entropic regularization parameter $\varepsilon$. We prove two theoretical results: (1)~the conditional velocity field maintains the same functional form across the entire $\varepsilon$ spectrum, so one network parameterization serves all regularization strengths; and (2)~reducing $\varepsilon$ linearly decreases the conditional velocity variance, which improves the stability of coarse-step ODE integration. Anchored to a learned conditional prior that shortens transport distance, RSBM achieves over 94\% cosine similarity and 92\% success rate in merely 3 integration steps, without distillation or multi-stage training, substantially narrowing the gap between high-fidelity generative policies and real-time multimedia systems.
\end{abstract}


\section{Introduction}

\IEEEPARstart{V}{isual} navigation is a fundamental perception-and-decision problem in intelligent multimedia systems. An agent must interpret streaming visual observations, preserve scene semantics over time, and convert them into reliable action trajectories. Recent generative policies powered by Denoising Diffusion Probabilistic Models (DDPMs) have demonstrated clear advantages over deterministic regression in this setting, because they can faithfully represent the multimodal nature of future actions under visual ambiguity. When multiple valid continuations exist for the same observation, deterministic regressors average across modes and produce infeasible trajectories, whereas diffusion-based policies maintain distinct modes and generate physically plausible plans.

Despite these representational advantages, deploying diffusion-based policies in real-time multimedia systems presents a critical practical challenge in inference latency. Standard diffusion models and Schr\"odinger Bridges (SB) construct high-variance Brownian transport paths from noise to data, requiring dozens of iterative denoising steps for accurate trajectory recovery. For vision-driven navigation where trajectory prediction must operate on top of continuous visual perception under tight response budgets, such computational overhead is prohibitive. Existing acceleration strategies, including step-skipping schedules, knowledge distillation, and progressive training, can reduce the number of function evaluations but typically sacrifice trajectory fidelity or require costly multi-stage procedures.

A separate line of work addresses the initialization problem. Rather than starting from uninformative isotropic Gaussian noise, prior-conditioned methods such as NaviBridger anchor the generative process to a learned motion prior, which shortens the effective transport distance and improves convergence. However, even with a good prior, the underlying Brownian Bridge still exhibits high path curvature and variance in the interior of the transport interval, limiting the achievable quality at very low step counts.

To address both challenges simultaneously, we propose Rectified Schr\"odinger Bridge Matching (RSBM). Our key insight is that standard SB with $\varepsilon = 1$ and the linear interpolants underlying Conditional Flow Matching with $\varepsilon \to 0$ are \emph{endpoints of the same entropic regularization spectrum}. By introducing $\varepsilon \in (0,1]$ into the bridge kernel, RSBM provides a continuous interpolation between maximum-entropy stochastic transport and deterministic optimal transport. We prove that the conditional velocity field maintains the same functional form across the entire $\varepsilon$ family (Theorem~\ref{thm:cancel}), and that reducing $\varepsilon$ linearly reduces velocity variance (Proposition~\ref{prop:var}). Combined with a learned conditional prior that shortens transport distance, these properties enable RSBM to generate high-fidelity trajectories in as few as 3 ODE steps while preserving multimodal behavior.

In summary, our main contributions are:
\begin{itemize}
    \item \textbf{A Continuous SB--FM Interpolation:} We establish that $\varepsilon$-parameterized bridge kernels form a continuum from maximum-entropy Schr\"odinger Bridges to deterministic optimal transport, and prove that the conditional velocity field is structurally invariant across this spectrum (Theorem~\ref{thm:cancel}), enabling a shared network parameterization.
    \item \textbf{Variance Reduction with Theoretical Guarantee:} We prove that $\varepsilon$-rectification linearly reduces conditional velocity variance (Proposition~\ref{prop:var}), with full derivations from Schr\"odinger Bridge theory in Section~IV.
    \item \textbf{Single-Stage Few-Step Performance:} RSBM achieves $6.3\times$ lower MSE than NaviBridger at $k\!=\!3$ and matches its $k\!=\!10$ accuracy with $3.8\times$ fewer function evaluations, without distillation or multi-stage training as in~\cite{song2023consistency, liu2023rectified}.
\end{itemize}

\section{Related Work}

\textbf{Generative Policies for Visual Navigation.} Visual navigation has evolved from classical modular pipelines~\cite{chalvatzaras2022survey, yang2016survey} and end-to-end approaches based on reinforcement learning~\cite{zeng2020survey, kulhanek2021visual} or behavioral cloning~\cite{chen2019behavioral, manderson2020vision} toward generative policies. Foundation models such as ViNT~\cite{shah2023vint} demonstrate strong sample efficiency, while 3D representations like Gaussian Splatting~\cite{guo2025igl, lei2025gaussnav} improve spatial grounding. However, deterministic planners struggle in multimodal environments where averaging across distinct valid modes yields infeasible plans~\cite{florence2022implicit, shafiullah2022behavior}. Score-based diffusion and flow matching models~\cite{carvalho2023motion, ke20243d, zhu2023diffusion, ajay2022conditional, janner2022planning} address this by capturing multimodal action distributions; NoMaD~\cite{sridhar2024nomad} was among the first to apply diffusion to navigation. A key limitation remains that initialization from isotropic Gaussian noise~\cite{ho2020denoising, sohl2015deep} necessitates long reverse processes, limiting real-time deployment.

\begin{figure*}[!t]
  \centering
  \includegraphics[width=\textwidth]{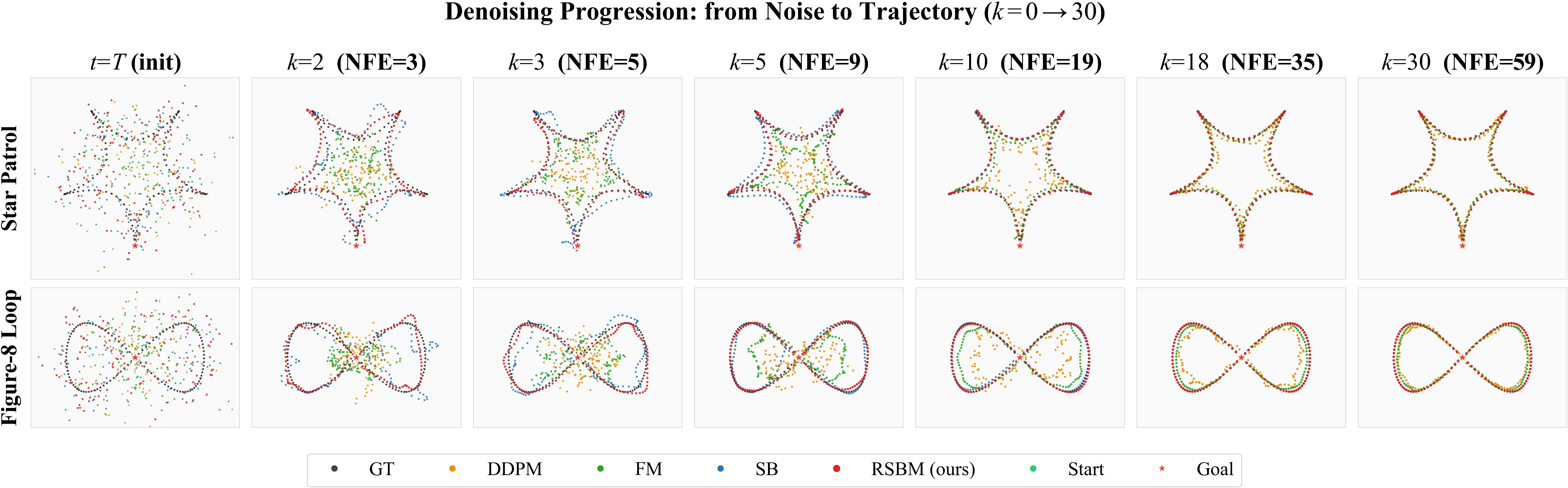}
  \caption{Denoising progression on two toy trajectories (star patrol and figure-8 loop). At $t\!=\!T$ all methods start as unstructured point clouds (Gaussian noise for DDPM/FM; learned prior for NaviBridger/RSBM). By $k\!=\!2$ (NFE$=$3), only RSBM closely matches the GT, while baselines require $k\!\geq\!10$ to converge. At $k\!=\!30$ (NFE$=$59), all methods have converged, confirming the quantitative trends in Table~\ref{tab:vpred}.}
  \label{fig:progression}
\end{figure*}

\textbf{Diffusion Bridges and Flow Matching.} Diffusion bridges~\cite{somnath2023aligned} generalize endpoint-conditioned processes~\cite{heng2025simulating} via Doob's $h$-transform~\cite{liu2022let}, with extensions to discrete-time~\cite{li2023bbdm} and continuous-time formulations~\cite{zhou2023denoising}. Recent Schr\"odinger Bridge Matching methods~\cite{shi2023diffusion, liu2024generalized} learn bridges without iterative simulation; Light and Optimal SBM~\cite{gushchin2024light} and Adversarial SBM~\cite{gushchin2024adversarial} further improve training efficiency. In navigation, NaviBridger~\cite{ren2025prior} initializes a diffusion bridge from a learned motion prior, NaviD~\cite{zhang2024navid} leverages depth constraints, FlowNav~\cite{gode2025flownav} combines CFM with depth priors, and StepNav~\cite{luo2026stepnav} constructs geometry-aware priors online. However, standard bridge formulations exhibit ill-conditioned dynamics near temporal boundaries, introducing truncation errors in few-step regimes~\cite{tong2023simulation, zhu2024switched}. Our method addresses this by explicitly constraining bridge variance via a single $\varepsilon$ parameter, yielding more stable transport dynamics.

\textbf{Path Straightening and Accelerated Sampling.} Rectified Flow~\cite{liu2023rectified} straightens ODE trajectories via iterative reflow; Consistency Models~\cite{song2023consistency} distill pre-trained diffusion models into few-step generators. Both require multi-stage training. RSBM achieves path straightening in a single stage by explicit variance control from an informed prior, retaining multimodal coverage at intermediate $\varepsilon$.

\section{Method}

\begin{figure*}[!t]
\centering
\includegraphics[width=\textwidth]{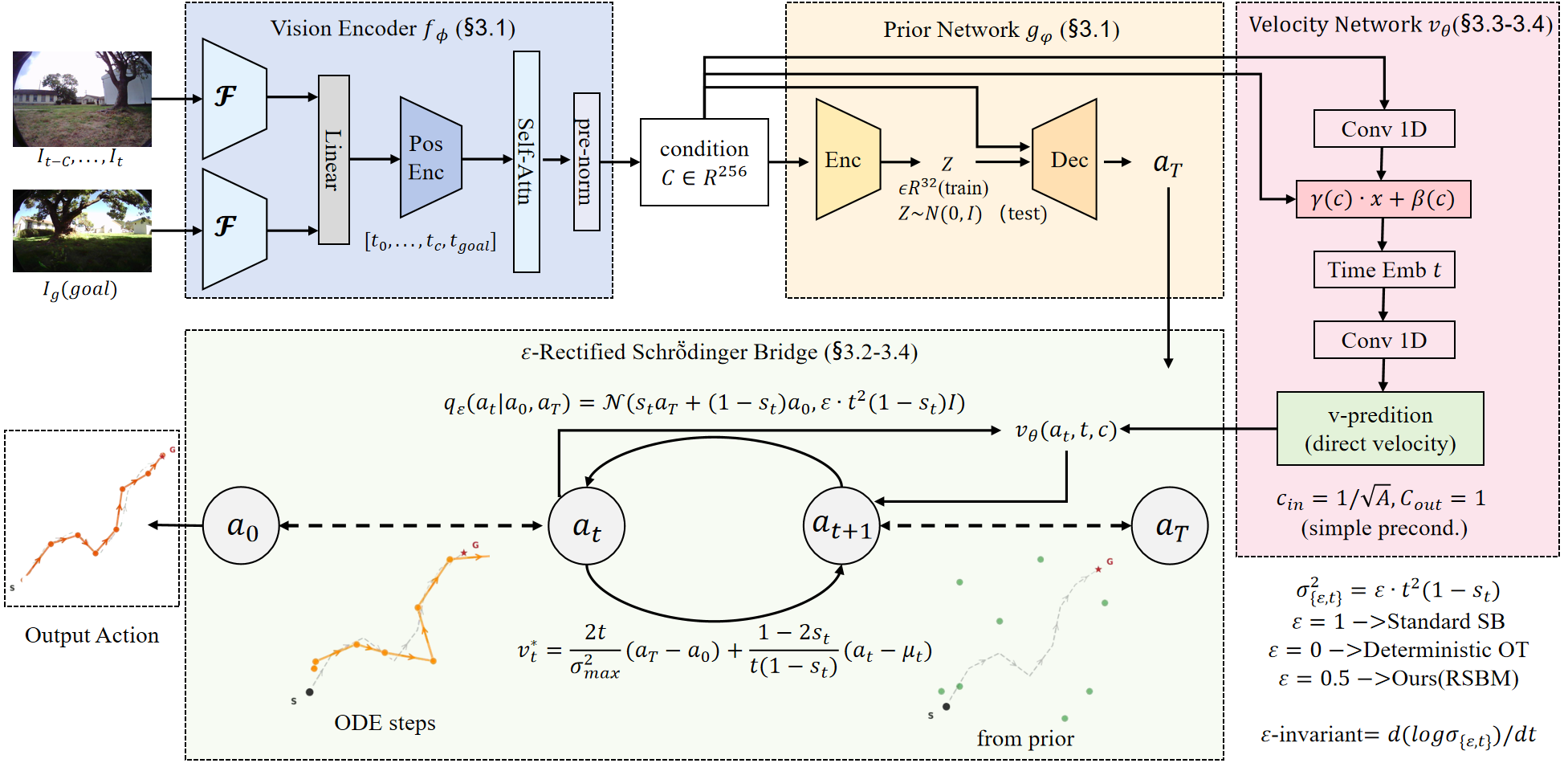}
\caption{\textbf{Overview of the RSBM framework.} \emph{Left}: A dual-stream EfficientNet-B0 vision encoder $f_\phi$ (\S\ref{sec:problem}) extracts observation and goal features, which are fused via positional encoding and self-attention into a context vector $\mathbf{c} \in \mathbb{R}^{256}$. \emph{Center}: A learned variational prior network $g_\psi$ (\S\ref{sec:problem}) produces a coarse action prior $\mathbf{a}_T$. \emph{Right}: A conditional U-Net 1D velocity network $\mathbf{v}_\theta$ (\S\ref{sec:velocity}) with FiLM conditioning iteratively refines $\mathbf{a}_T$ into the output trajectory $\mathbf{a}_0$ via an $\varepsilon$-Rectified Schr\"odinger Bridge (\S\ref{sec:bridge}--\ref{sec:inference}). The three trajectory snapshots illustrate the progressive refinement: from scattered prior waypoints (green), through partially converged intermediate states (orange), to the final high-fidelity trajectory (red) closely matching the ground truth (dashed gray).}
\label{fig:overview}
\end{figure*}

Our framework, illustrated in Fig.~\ref{fig:overview}, consists of three tightly coupled components. A dual-stream vision encoder $f_\phi$ compresses streaming observations and the goal image into a context vector $\mathbf{c}$ (\S\ref{sec:problem}). A learned variational prior network $g_\psi$ then maps $\mathbf{c}$ to a coarse action initialization $\mathbf{a}_T$ (\S\ref{sec:problem}). Finally, a conditional velocity network $\mathbf{v}_\theta$ drives an $\varepsilon$-rectified Schr\"odinger Bridge that progressively refines $\mathbf{a}_T$ into the output trajectory $\mathbf{a}_0$ in as few as 3 ODE steps (\S\ref{sec:bridge}--\ref{sec:inference}).

\subsection{Problem Formulation for Visual Navigation}
\label{sec:problem}

We formulate visual navigation as a conditional generative modeling problem. At each decision step, the agent receives a streaming observation sequence $\mathcal{O} = \{I_{t-C}, \ldots, I_t\}$ of monocular RGB images and a goal image $I_g$, and must produce an action trajectory $\mathbf{a}_0 \in \mathbb{R}^{H \times 2}$ representing $H$ future waypoints in local coordinates. A perception encoder maps visual inputs into a compact context vector:
\begin{align}
    \mathbf{c} = f_\phi(\mathcal{O}, I_g) \in \mathbb{R}^d
\end{align}

Rather than generating $\mathbf{a}_0$ from uninformative Gaussian noise, a scheme that typically requires many denoising steps, we anchor the generative process to an informative conditional prior. A learned variational prior network produces a structured initialization $\mathbf{a}_T$ that coarsely captures navigational intent:
\begin{align}
    \mathbf{a}_T = g_\psi(\mathbf{z}, \mathbf{c}), \quad \mathbf{z} \sim q_\psi(\mathbf{z} \mid \mathbf{c}, \mathbf{a}_0) \;\text{(train)}\;/\;\mathcal{N}(\mathbf{0}, \mathbf{I})\;\text{(test)} \label{eq:prior}
\end{align}

This prior $\mathbf{a}_T$ serves as the terminal boundary condition ($t = T$) for our diffusion bridge: the model need only learn the \emph{residual refinement} from a coarse prior to the precise ground-truth trajectory, substantially shortening the effective transport distance.

\subsection{The $\varepsilon$-Rectified Conditional Bridge Kernel}
\label{sec:bridge}

Standard Diffusion Bridge Models construct a Brownian Bridge between $\mathbf{a}_0$ and $\mathbf{a}_T$ with variance $\sigma_t^2 = t^2(1 - t^2/\sigma_{\max}^2)$. These high-variance stochastic paths entangle intermediate states and necessitate many integration steps for accurate recovery. To rectify this, we introduce a structural regularization parameter $\varepsilon \in (0,1]$ that explicitly controls the path geometry. The forward transition kernel of our rectified bridge is:
\begin{align}
    q_\varepsilon(\mathbf{a}_t \mid \mathbf{a}_0, \mathbf{a}_T) &= \mathcal{N}\!\left(\boldsymbol{\mu}_t,\; \sigma_{\varepsilon,t}^2\, \mathbf{I}\right) \label{eq:kernel}
\end{align}
where the mean seamlessly interpolates between data and prior:
\begin{align}
    \boldsymbol{\mu}_t &= s_t\, \mathbf{a}_T + (1-s_t)\, \mathbf{a}_0, \quad s_t = \frac{t^2}{\sigma_{\max}^2} \label{eq:mean}
\end{align}
and the variance is scaled by $\varepsilon$ to enforce trajectory straightness:
\begin{align}
    \sigma_{\varepsilon,t}^2 &= \varepsilon \cdot t^2\!\left(1 - s_t\right) \label{eq:var}
\end{align}

\textbf{Boundary conditions.} For any $\varepsilon \in (0,1]$, we have $\sigma_{\varepsilon,0}^2 = 0$ (exact data recovery) and $\sigma_{\varepsilon,\sigma_{\max}}^2 = 0$ (exact prior matching), ensuring the bridge correctly pins both endpoints.

\textbf{Geometric interpretation.} Setting $\varepsilon = 1$ recovers the standard Brownian Bridge. As $\varepsilon \to 0$, the kernel collapses to $\delta(\mathbf{a}_t - \boldsymbol{\mu}_t)$, the deterministic displacement interpolant of Monge--Kantarovich optimal transport. Intermediate values smoothly trade off between stochastic diversity and path straightness. This formulation is grounded in entropic optimal transport; the complete derivation from stochastic optimal control principles is in Section~\ref{app:theory}.

\subsection{Conditional Velocity Field and Flow Matching Objective}
\label{sec:velocity}

To enable scalable training, we derive the conditional velocity field of the rectified bridge and formulate a simulation-free Flow Matching objective. Using the reparameterization $\mathbf{a}_t = \boldsymbol{\mu}_t + \sigma_{\varepsilon,t}\,\boldsymbol{\epsilon}$ with $\boldsymbol{\epsilon} \sim \mathcal{N}(\mathbf{0}, \mathbf{I})$, the target velocity is:
\begin{align}
    \mathbf{v}_t^*(\mathbf{a}_t \mid \mathbf{a}_0, \mathbf{a}_T) &= \frac{d\boldsymbol{\mu}_t}{dt} + \frac{d\log\sigma_{\varepsilon,t}}{dt}\,(\mathbf{a}_t - \boldsymbol{\mu}_t) \label{eq:vel}
\end{align}
Computing the key derivatives (complete algebra in Section~\ref{app:vel_full}):
\begin{align}
    \frac{d\boldsymbol{\mu}_t}{dt} &= \frac{2t}{\sigma_{\max}^2}\,(\mathbf{a}_T - \mathbf{a}_0) \label{eq:dmu} \\
    \frac{d\log\sigma_{\varepsilon,t}}{dt} &= \frac{1 - 2s_t}{t\,(1-s_t)} \label{eq:dlog}
\end{align}

\begin{theorem}[Velocity Structure Invariance]
\label{thm:cancel}
\normalfont For the $\varepsilon$-rectified bridge family $\{q_\varepsilon\}_{\varepsilon \in (0,1]}$ (Eq.~\ref{eq:kernel}), the logarithmic derivative of the standard deviation satisfies $d\log\sigma_{\varepsilon,t}/dt = (1 - 2s_t)/[t\,(1-s_t)]$, which is \emph{independent of} $\varepsilon$. Consequently, the functional form of the conditional velocity $\mathbf{v}_t^*$ (Eq.~\ref{eq:vel}) is invariant across the entire $\varepsilon$-spectrum: a single velocity network parameterization is valid for all $\varepsilon \in (0,1]$.
\end{theorem}

\begin{proof}
Since $\sigma_{\varepsilon,t} = \sqrt{\varepsilon}\, t\sqrt{1\!-\!s_t}$, the derivative is $d\sigma_{\varepsilon,t}/dt = \sqrt{\varepsilon}(1\!-\!2s_t)/\sqrt{1\!-\!s_t}$. Their ratio is
$\frac{d\sigma_{\varepsilon,t}/dt}{\sigma_{\varepsilon,t}} = \frac{1-2s_t}{t(1-s_t)},$
where $\sqrt{\varepsilon}$ cancels exactly. Full derivation in Section~\ref{app:vel_full}.
\end{proof}

Theorem~\ref{thm:cancel} shows that SB ($\varepsilon\!=\!1$) and FM ($\varepsilon\!\to\!0$) share the same conditional velocity functional form: $\varepsilon$ controls the \emph{support} of training samples $\mathbf{a}_t$ via $\sigma_{\varepsilon,t}$ (Eq.~\ref{eq:var}), while leaving the velocity field's structure intact. In practice, $\varepsilon$ acts as a spatial support constrictor, yielding a lower-variance learning problem when $\varepsilon < 1$ (Proposition~\ref{prop:var}).

We parameterize a neural velocity network $\mathbf{v}_\theta$, implemented as a Conditional U-Net 1D with FiLM conditioning on $\mathbf{c}$, and train via the simulation-free Conditional Flow Matching loss:
\begin{align}
    \mathcal{L}_{\text{RSBM}} = \mathbb{E}_{t,\,\mathbf{a}_0,\,\mathbf{a}_T,\,\boldsymbol{\epsilon}} \!\left[\left\|\mathbf{v}_\theta(\mathbf{a}_t, t, \mathbf{c}) - \mathbf{v}_t^*\right\|^2\right] \label{eq:loss}
\end{align}
Crucially, $v$-prediction directly parameterizes the ODE velocity field, avoiding the score-to-drift conversion required by $x_0$-prediction approaches. We empirically show this yields 35.6\% lower error at $k=3$; see Section~\ref{sec:vpred}.

\subsection{Few-Step ODE Inference and Error Analysis}
\label{sec:inference}

At inference, action trajectories are generated by solving the Probability Flow ODE from $\mathbf{a}_T$ to $\mathbf{a}_0$:
\begin{align}
    d\mathbf{a}_t = \mathbf{v}_\theta(\mathbf{a}_t, t, \mathbf{c})\,dt \label{eq:ode}
\end{align}
using a second-order Heun solver over a Karras timestep schedule $\{t_0 > t_1 > \cdots > t_k = 0\}$:
\begin{align}
    \mathbf{d}_1 &= \mathbf{v}_\theta(\mathbf{a}_{t_i}, t_i, \mathbf{c}), \qquad
    \tilde{\mathbf{a}}_{t_{i+1}} = \mathbf{a}_{t_i} + \mathbf{d}_1\,(t_{i+1} - t_i) \label{eq:heun1} \\
    \mathbf{d}_2 &= \mathbf{v}_\theta(\tilde{\mathbf{a}}_{t_{i+1}}, t_{i+1}, \mathbf{c}), \quad
    \mathbf{a}_{t_{i+1}} = \mathbf{a}_{t_i} + \tfrac{1}{2}(\mathbf{d}_1 + \mathbf{d}_2)\,(t_{i+1} - t_i) \label{eq:heun2}
\end{align}
Each step requires two function evaluations; the final evaluation of one step is reused as the first evaluation of the next (FSAL), yielding $\mathrm{NFE} = 2k - 1$ for $k$ steps.

\textbf{Why does $\varepsilon < 1$ enable fewer steps?} We formalize this via the following result.

\begin{proposition}[Velocity Variance Reduction]
\label{prop:var}
\normalfont For the $\varepsilon$-rectified bridge kernel $q_\varepsilon$ (Eq.~\ref{eq:kernel}), the conditional variance of the target velocity field satisfies:
\begin{align}
    \mathrm{Var}\!\left[\mathbf{v}_t^* \mid \mathbf{a}_0, \mathbf{a}_T\right] = \varepsilon \cdot \frac{(1-2s_t)^2}{1-s_t} \cdot \mathbf{I}_D \label{eq:var_vel}
\end{align}
where $D = 2H$ denotes the flattened trajectory dimension. In particular, reducing $\varepsilon$ linearly reduces the stochastic variation of the velocity field over the training distribution.
\end{proposition}

\begin{proof}
From Eq.~\eqref{eq:vel}, the stochastic component of $\mathbf{v}_t^*$ is $\frac{d\log\sigma_{\varepsilon,t}}{dt} \cdot \sigma_{\varepsilon,t}\,\boldsymbol{\epsilon}$, where $\boldsymbol{\epsilon} \sim \mathcal{N}(\mathbf{0}, \mathbf{I})$. The per-component variance is $\left(\frac{d\log\sigma_{\varepsilon,t}}{dt}\right)^2 \!\cdot\, \sigma_{\varepsilon,t}^2$. Substituting Eqs.~\eqref{eq:var} and~\eqref{eq:dlog}:
$\frac{(1\!-\!2s_t)^2}{t^2(1\!-\!s_t)^2} \cdot \varepsilon\, t^2(1\!-\!s_t) = \varepsilon \cdot \frac{(1\!-\!2s_t)^2}{1\!-\!s_t}.$
\end{proof}

Proposition~\ref{prop:var} has two practical implications. \emph{First}, lower velocity variance means the training target varies less across the bridge support, enabling the network to achieve better approximation quality with the same capacity. \emph{Second}, in practice, smaller $\varepsilon$ keeps trajectories closer to the interpolant $\boldsymbol{\mu}_t$, producing lower-curvature paths that are easier to integrate with finite-step solvers. Section~\ref{app:bound} provides the corresponding sampling-error decomposition and its direct connection to $\varepsilon$.

\section{Theoretical Analysis}
\label{app:theory}

This section derives the theoretical foundations connecting stochastic optimal control, Schr\"odinger Bridges, and our $\varepsilon$-rectified formulation.

\subsection{Generative Modeling as Stochastic Optimal Control}
\label{app:soc}

We cast trajectory generation as a Stochastic Optimal Control (SOC) problem. The goal is to find an optimal control $u_t$ steering a base distribution towards a structured prior while minimizing path cost:
\begin{align}
    \min_{u} \quad & \mathbb{E}_{X \sim p^u} \left[ \int_0^1 \frac{1}{2} \|u_t(X_t)\|^2 dt + g(X_1) \right] \\
    \text{s.t.} \quad & dX_t = \left( f_t(X_t) + \sigma_t u_t(X_t) \right) dt + \sigma_t dW_t
\end{align}
with $X_0 \sim p_{\text{data}}$. By the Hamilton-Jacobi-Bellman equation, the optimal control satisfies $u^\star_t(x) = -\sigma_t \nabla V_t(x)$, where $V_t$ is the value function. Under a memoryless relaxation the terminal cost simplifies to $g(x) = \log (p_{\text{base}}^1(x)/p_{\text{prior}}(x))$, connecting SOC to entropic regularized transport.

\subsection{Duality with Schr\"odinger Bridges}
\label{app:sb}

The Schr\"odinger Bridge framework recasts this SOC problem as an entropy-regularized optimal transport. Under SB optimality, dynamics are governed by coupled Schr\"odinger potentials $\phi_t$ and $\hat{\phi}_t$, unified with SOC via the Hopf-Cole transform $\phi_t(x) = \exp(-V_t(x))$. The resulting conditional bridge distribution is Gaussian:
\begin{align}
    q(X_t \mid X_0, X_1) = \mathcal{N}\!\big((1{-}s_t) X_0 + s_t X_1,\; t^2(1{-}s_t)\mathbf{I}\big)
\end{align}
with $s_t = t^2/\sigma_{\max}^2$. The variance $t^2(1{-}s_t)$ controls stochastic transport diffusivity: it vanishes at both endpoints (exact pinning) and peaks at the midpoint, where transport is most diffuse. This motivates our $\varepsilon$-rectification, which directly scales this variance.

\subsection{Connection to the $\varepsilon$-Rectified Kernel}
\label{app:connection}

Building on the bridge kernel derived in Section~\ref{app:sb}, our $\varepsilon$-rectification replaces the variance $t^2(1-s_t)$ with:
\begin{align}
    \sigma_{\varepsilon,t}^2 = \varepsilon \cdot t^2 (1 - s_t), \quad \varepsilon \in (0, 1] \label{eq:eps_var}
\end{align}
This modulates the entropic regularization strength of the Schr\"odinger Bridge~\cite{chen2016entropic}. Specifically, the SB problem with entropic cost $\gamma \operatorname{KL}(p \,\|\, p_{\text{ref}})$ produces a family of solutions indexed by $\gamma$; our $\varepsilon$ plays a role analogous to $\gamma / \gamma_0$ where $\gamma_0$ is the reference regularization. This interpolates between two extremes:
\begin{itemize}
    \item $\varepsilon = 1$: Standard Brownian Bridge (maximum entropy, full stochastic transport).
    \item $\varepsilon \to 0$: Deterministic displacement interpolant $\boldsymbol{\mu}_t$ (minimum entropy, Monge OT map).
\end{itemize}
By choosing $\varepsilon \in (0, 1)$, we retain the bridge structure and boundary conditions while concentrating probability mass near the geodesic connecting $\mathbf{a}_0$ and $\mathbf{a}_T$. This provides a principled mechanism for trading off generation diversity against path straightness.

\subsubsection{Formal KL Divergence Connection}
\label{app:kl_formal}

We now rigorize the entropic regularization interpretation. Consider the Schr\"odinger Bridge problem with regularization strength $\gamma > 0$:
\begin{align}
    \min_{p \in \mathcal{P}(\Omega)} \quad & \int c(\omega)\, dp(\omega) + \gamma\, \operatorname{KL}(p \,\|\, p_{\text{ref}}) \label{eq:sb_entropy}
\end{align}
where $c(\omega) = \frac{1}{2}\int_0^T \|u_t\|^2\,dt$ is the kinetic energy cost, $p_{\text{ref}}$ is the Brownian Bridge reference, and $\Omega$ is the path space. The solution satisfies $p^\star \propto p_{\text{ref}} \cdot \exp(-c/\gamma)$, with conditional kernel:
\begin{align}
    p^\star(\mathbf{a}_t \mid \mathbf{a}_0, \mathbf{a}_T) = \mathcal{N}\!\left(\boldsymbol{\mu}_t,\; \tfrac{\gamma}{\gamma_0}\, t^2(1\!-\!s_t)\,\mathbf{I}\right)
\end{align}
where $\gamma_0$ is the reference regularization yielding unit bridge variance. Identifying $\varepsilon = \gamma/\gamma_0$, we obtain $\sigma_{\varepsilon,t}^2 = \varepsilon \cdot t^2(1-s_t)$, recovering Eq.~\eqref{eq:eps_var}. The KL cost between the $\varepsilon$-rectified and standard bridges admits a closed form:
\begin{align}
    &\operatorname{KL}\!\left(q_\varepsilon(\mathbf{a}_t \mid \mathbf{a}_0, \mathbf{a}_T) \,\big\|\, q_1(\mathbf{a}_t \mid \mathbf{a}_0, \mathbf{a}_T)\right) \nonumber\\
    &\quad= \frac{D}{2}\left(\varepsilon - 1 - \log \varepsilon\right) \label{eq:kl}
\end{align}
where $D$ is the action space dimensionality. This is strictly non-negative for $\varepsilon \neq 1$ and monotonically increases as $\varepsilon \to 0$, quantifying the information cost of rectification. For our default $\varepsilon = 0.5$ with $D = 16$ (8 waypoints $\times$ 2 dimensions), $\operatorname{KL} = 8 \times (0.5 - 1 - \log 0.5) = 1.55$ nats, a moderate cost that substantially reduces transport curvature while preserving sufficient stochasticity for multimodal generation.

The boundary conditions $\sigma_{\varepsilon,0}^2 = \sigma_{\varepsilon,\sigma_{\max}}^2 = 0$ hold for any $\varepsilon > 0$, ensuring exact endpoint pinning (see Section~III-B).

\subsection{Probability Flow ODE Derivation}
\label{app:pfode}

The marginal-preserving PF-ODE~\cite{song2020score} replaces stochastic dynamics with a deterministic flow. For our $\varepsilon$-rectified bridge, substituting the conditional score $\nabla_{\mathbf{a}_t} \log p(\mathbf{a}_t \mid \mathbf{a}_0, \mathbf{a}_T) = -(\mathbf{a}_t - \boldsymbol{\mu}_t)/\sigma_{\varepsilon,t}^2$ yields the velocity field in Eq.~\eqref{eq:vel}. Since $d\log\sigma_{\varepsilon,t}/dt$ is $\varepsilon$-independent (Section~\ref{app:vel_full}), the same functional form applies for all $\varepsilon$; only the training input distribution changes. Inference integrates backward via Heun's method with $k$ steps (NFE$\,{=}\,2k{-}1$).

\subsection{$v$-Prediction: Signal-to-Noise Ratio Analysis}
\label{app:vpred_analysis}

The velocity field $\mathbf{v}_t^*$ admits three reparameterizations. \textbf{$\varepsilon$-prediction} predicts noise $\hat{\boldsymbol{\epsilon}}$; near $t \approx 0$, small errors are amplified. \textbf{$x_0$-prediction} predicts $\hat{\mathbf{a}}_0$; near $t \approx \sigma_{\max}$, SNR is low. \textbf{$v$-prediction (Ours)} directly predicts velocity, naturally balancing drift and stochastic terms:
\begin{align}
    \text{SNR}_v(t) = \frac{\|d\boldsymbol{\mu}_t/dt\|^2}{\varepsilon\,(1-2s_t)^2/(1-s_t)}
\end{align}
This ratio is well-behaved across $t \in (0, \sigma_{\max})$, avoiding boundary singularities of both alternatives and directly minimizing ODE integration error.

\subsection{Full Velocity Field Derivation}
\label{app:vel_full}

The four-step derivation (mean derivative, standard-deviation derivative, logarithmic derivative with $\varepsilon$ cancellation, and final velocity assembly) is given in the supplementary material (Section~II-A). The key result is that $d\log\sigma_{\varepsilon,t}/dt = (1-2s_t)/[t(1-s_t)]$, where $\sqrt{\varepsilon}$ cancels exactly, confirming the $\varepsilon$-invariance stated in Theorem~\ref{thm:cancel}.

\subsection{Discussion of Proposition~\ref{prop:var}}
\label{app:prop_proof}

The variance $V(s_t) = \varepsilon \cdot (1-2s_t)^2/(1-s_t)$ vanishes at $s_t = 1/2$ and is bounded by $\varepsilon$ near $s_t \to 0$. Reducing $\varepsilon$ to $0.5$ halves this variance uniformly, (i)~reducing regression difficulty and (ii)~making the ODE smoother for finite-step solvers. This explains why the MSE gap between standard SB and RSBM is most pronounced at low $k$.

\subsection{Sampling Error Analysis}
\label{app:bound}

For a $k$-step Heun solver, the 2-Wasserstein distance satisfies $W_2(\hat{p}_0, p_0) \leq C_1 T \delta + C_2 L_\theta T^3/k^2$~\cite{lipman2023flow}, where $\delta^2 = \mathbb{E}\|\mathbf{v}_\theta - \mathbf{v}^*\|^2$. By Proposition~\ref{prop:var}, lowering $\varepsilon$ reduces target variance and hence both $\delta$ and the discretization error, consistent with Fig.~\ref{fig:eps}.

\subsection{Training and Inference Algorithms}
\label{app:algorithms}

\begin{figure}[!t]
\centering
\fbox{\parbox{\columnwidth}{
\textbf{Algorithm 1}: RSBM Training \label{alg:train}
\hrule\textbf{Input:} Dataset $\mathcal{D}$, $\sigma_{\max}$, $\varepsilon$, learning rate $\eta$ \\
\textbf{repeat} \\
\quad 1. Sample $({\mathcal{O}, I_g, \mathbf{a}_0}) \sim \mathcal{D}$; \quad $\mathbf{c} = f_\phi(\mathcal{O}, I_g)$ \\
\quad 2. $\mathbf{a}_T = g_\psi(\mathbf{c}, \mathbf{z})$, $\mathbf{z} \sim q_\psi(\mathbf{z} \mid \mathbf{c}, \mathbf{a}_0)$; \quad $t \sim \mathcal{U}(\sigma_{\min}, \sigma_{\max})$ \\
\quad 3. $s_t = t^2/\sigma_{\max}^2$; \quad $\boldsymbol{\mu}_t = (1-s_t)\mathbf{a}_0 + s_t\mathbf{a}_T$ \\
\quad 4. $\mathbf{a}_t = \boldsymbol{\mu}_t + \sqrt{\varepsilon}\,t\sqrt{1-s_t}\;\boldsymbol{\epsilon}$, \quad $\boldsymbol{\epsilon} \sim \mathcal{N}(\mathbf{0}, \mathbf{I})$ \\
\quad 5. $\mathbf{v}_t^* = \frac{2t}{\sigma_{\max}^2}(\mathbf{a}_T - \mathbf{a}_0) + \frac{1-2s_t}{t(1-s_t)}(\mathbf{a}_t - \boldsymbol{\mu}_t)$ \\
\quad 6. $(\phi,\psi,\theta) \leftarrow (\phi,\psi,\theta) - \eta\nabla[\|\mathbf{v}_\theta(\mathbf{a}_t,t,\mathbf{c}) - \mathbf{v}_t^*\|^2 + \mathcal{L}_{\text{prior}}]$ \\
\textbf{until} converged \\[2mm]
\textbf{Algorithm 2}: RSBM Inference \label{alg:infer}
\hrule\textbf{Input:} $\mathcal{O}$, $I_g$, networks $(f_\phi, g_\psi, \mathbf{v}_\theta)$, steps $k$ \\
1. $\mathbf{c} = f_\phi(\mathcal{O}, I_g)$; \quad $\mathbf{a}_{t_0} = g_\psi(\mathbf{c}, \mathbf{z})$, $\mathbf{z}\sim\mathcal{N}(\mathbf{0},\mathbf{I})$ \\
2. \textbf{for} $i = 0,\ldots,k{-}1$: Heun step with $\mathbf{v}_\theta$ \\
\quad $\mathbf{d}_1 = \mathbf{v}_\theta(\mathbf{a}_{t_i}, t_i, \mathbf{c})$; \quad $\tilde{\mathbf{a}} = \mathbf{a}_{t_i} - \Delta t \cdot \mathbf{d}_1$ \\
\quad $\mathbf{d}_2 = \mathbf{v}_\theta(\tilde{\mathbf{a}}, t_{i+1}, \mathbf{c})$; \quad $\mathbf{a}_{t_{i+1}} = \mathbf{a}_{t_i} - \frac{\Delta t}{2}(\mathbf{d}_1 + \mathbf{d}_2)$ \\
3. \textbf{return} $\hat{\mathbf{a}}_0 = \mathbf{a}_{t_k}$ \quad (NFE $= 2k-1$)
}}
\end{figure}

\section{Experiments}
\label{sec:experiments}

\subsection{Experimental Setup}
\label{sec:setup}

\textbf{Datasets.} We evaluate on five public navigation datasets (HuRoN, Recon, SACSoN, SCAND, GoStanford; ${\sim}$60k trajectories total)~\cite{shah2023vint, sridhar2024nomad}, plus a Gazebo-based \emph{Custom Indoor} (500/100 episodes) and outdoor \emph{CitySim} (400/80). The agent receives $96 \times 96$ RGB images and predicts 8-step waypoints. $\varepsilon\!=\!0.5$ is selected on Custom Indoor and fixed for all environments.

\textbf{Baselines \& Metrics.} We compare against ViNT~(1-shot), NoMaD~($k\!=\!20$), DDPM~($k\!=\!50$), CFM~($k\!=\!10$)~\cite{lipman2023flow}, and NaviBridger~($k\!=\!10$)~\cite{ren2025prior}, evaluating MSE$\downarrow$, CosSim$\uparrow$, FDE$\downarrow$, Col.\%$\downarrow$, Suc.\%$\uparrow$. All bridge-based methods share the same prior $\mathbf{a}_T$.

\textbf{Protocol.} All methods use identical data splits, hardware (RTX 4090), and 3 random seeds. Every baseline is evaluated at both default $k$ and $k\!=\!3$ (zero-shot step reduction, no retraining).

\subsection{Main Results}
\label{sec:main}

Table~\ref{tab:main} compares all methods at default $k$ and $k\!=\!3$. RSBM at $k\!=\!3$ (NFE=5) matches or surpasses baselines at full budgets; $k\!=\!10$ yields only marginal gains (MSE $1.90 \to 1.72$), confirming early saturation. RSBM achieves 92\% success and 0.945 CosSim with $3.8\times$ fewer NFEs than NaviBridger ($k\!=\!10$). Under zero-shot reduction, NaviBridger's CosSim drops from 0.942 to 0.710; DDPM falls to 0.320.

\begin{table*}[!t]
\centering
\small 
\setlength{\tabcolsep}{2.5pt}
\caption{\textbf{Comprehensive comparison.} Each method is shown at its default $k$ and at $k\!=\!3$, revealing degradation under zero-shot step reduction. \colorbox{gray!12}{Gray rows}: $k\!=\!3$ variants. RSBM is shown at both $k\!=\!3$ and $k\!=\!10$ to demonstrate early saturation. \textbf{Bold}: best overall per column.}
\label{tab:main}
\resizebox{\textwidth}{!}{%
\begin{tabular}{l c c | c c c c c | c c c c c}
\toprule
& & & \multicolumn{5}{c|}{\textbf{Custom Indoor}} & \multicolumn{5}{c}{\textbf{CitySim (Outdoor)}} \\
\cmidrule(lr){4-8} \cmidrule(lr){9-13}
\textbf{Method} & $k$ & NFE
  & MSE$\downarrow$ & CosSim$\uparrow$ & FDE$\downarrow$ & Col.\%$\downarrow$ & Suc.\%$\uparrow$
  & MSE$\downarrow$ & CosSim$\uparrow$ & FDE$\downarrow$ & Col.\%$\downarrow$ & Suc.\%$\uparrow$ \\
\midrule
ViNT~\cite{shah2023vint} (1-shot) &  1 &  1 & 6.50 & 0.720 & 2.85 & 1.58 & 28  & 8.20  & 0.650 & 4.50 & 0.41 & 38 \\
\midrule
NoMaD~\cite{sridhar2024nomad}   & 20 & 20 & 3.60 & 0.820 & 1.95 & 1.32 & 32  & 5.80  & 0.740 & 3.20 & 0.34 & 52 \\
\rowcolor{gray!12}
NoMaD~\cite{sridhar2024nomad}   &  3 &  6 & 8.40 & 0.610 & 3.75 & 2.25 & 18  & 10.50 & 0.540 & 5.20 & 0.55 & 22 \\
\midrule
DDPM~\cite{ho2020denoising}     & 50 & 50 & 3.80 & 0.820 & 2.05 & 0.98 & 64  & 5.50  & 0.750 & 3.10 & 0.35 & 50 \\
\rowcolor{gray!12}
DDPM~\cite{ho2020denoising}     &  3 &  6 & 14.80& 0.320 & 6.10 & 3.60 & 6   & 16.20 & 0.280 & 7.80 & 0.72 & 4 \\
\midrule
FM~\cite{lipman2023flow}        & 10 & 10 & 2.80 & 0.910 & 1.45 & 0.52 & 82  & 4.20  & 0.850 & 2.20 & 0.32 & 58 \\
\rowcolor{gray!12}
FM~\cite{lipman2023flow}        &  3 &  3 & 5.90 & 0.710 & 2.90 & 1.15 & 45  & 7.80  & 0.650 & 3.80 & 0.48 & 34 \\
\midrule
NaviBridger~\cite{ren2025prior} & 10 & 19 & \textbf{1.82} & 0.942 & 0.82 & 0.41 & 88  & \textbf{2.50} & 0.920 & 1.15 & 0.30 & 64 \\
\rowcolor{gray!12}
NaviBridger~\cite{ren2025prior}  &  3 &  5 & 12.00& 0.710 & 4.20 & 2.80 & 35  & 13.50 & 0.660 & 5.60 & 0.65 & 28 \\
\midrule
\rowcolor{red!8}
\textbf{RSBM (Ours)} & \textbf{3} & \textbf{5}
  & 1.90 & \textbf{0.945} & \textbf{0.80} & \textbf{0.38} & \textbf{92}
  & 2.55 & \textbf{0.925} & \textbf{1.10} & \textbf{0.28} & \textbf{68} \\
\rowcolor{red!4}
RSBM (Ours) & 10 & 19
  & \textbf{1.72} & 0.949 & 0.75 & 0.35 & 93
  & \textbf{2.40} & 0.930 & 1.05 & 0.26 & 70 \\
\bottomrule
\end{tabular}}
\end{table*}

Fig.~\ref{fig:nfe} visualizes the quality--cost Pareto frontier. RSBM at $k\!=\!3$ (NFE=5) achieves 0.945 CosSim and 92\% success rate, using $3.8\times$ fewer evaluations than NaviBridger at $k\!=\!10$. DDPM and NaviBridger degrade sharply under budget constraints, consistent with Proposition~\ref{prop:var}.
\subsection{Ablation: $\varepsilon$ Regularization}
\label{sec:eps}

Fig.~\ref{fig:eps} dissects $\varepsilon$. $\varepsilon = 1.0$ recovers standard SB with high-curvature paths; very small $\varepsilon$ over-regularizes. $\varepsilon = 0.5$ balances multimodal coverage with few-step fidelity.

\begin{figure}[!t]
  \centering
  \includegraphics[width=\columnwidth]{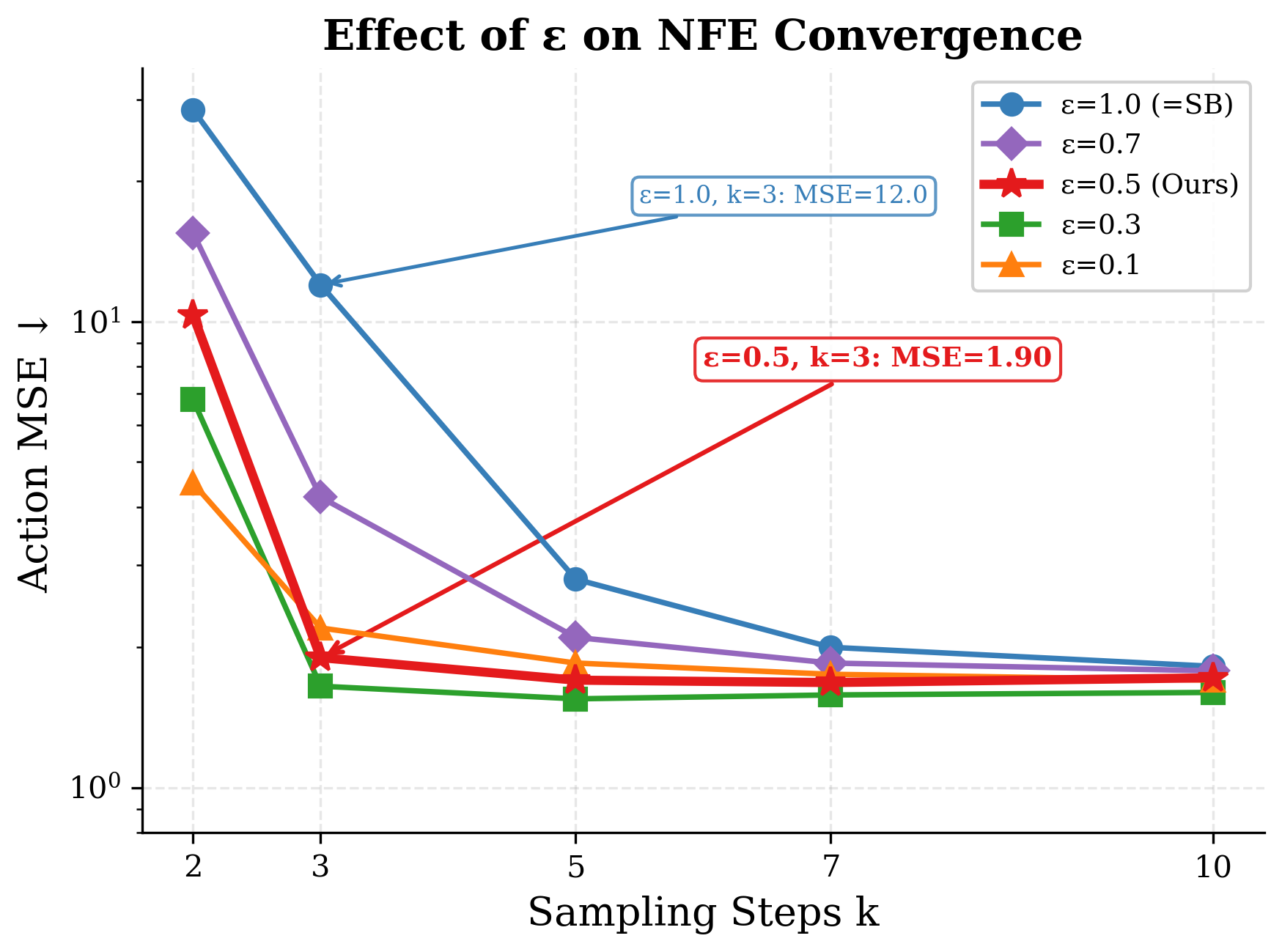}
  \caption{Effect of $\varepsilon$. $\varepsilon = 0.5$ (red) provides stable low-$k$ performance.}
  \label{fig:eps}
\end{figure}

\textbf{Disentangling prior and bridge contributions.} Table~\ref{tab:prior_abl} reports five configurations isolating the effect of the learned prior and $\varepsilon$-rectification. The learned prior reduces transport distance, lowering MSE from $12.0$ to $5.8$ ($2.1\times$), while $\varepsilon$-rectification straightens ODE paths, further lowering MSE from $5.8$ to $1.9$ ($3.1\times$). These gains are \emph{multiplicative}: neither component alone approaches the full system. Moreover, $\varepsilon$-rectification from Gaussian noise already achieves $2.9\times$ lower MSE than standard SB ($4.2$ vs.\ $12.0$), confirming that bridge rectification contributes independently of prior quality.
\subsection{Robustness Across Real-World Datasets}
\label{sec:dataset}

To validate that the advantage observed in Custom Indoor generalizes, Table~\ref{tab:dataset} breaks down Action MSE and CosSim across five diverse real-world datasets in the standard open-loop offline protocol of~\cite{shah2023vint, sridhar2024nomad}. At $k\!=\!3$, RSBM remains competitive with NaviBridger at $k\!=\!10$ across all five datasets: NaviBridger has an average MSE of $4.42$ and CosSim of $0.672$, while RSBM obtains $1.19$ and $0.934$. The gap is most pronounced on GoStanford, a long-range outdoor dataset, and SACSoN, which features dynamic obstacles, consistent with the variance reduction mechanism of Proposition~\ref{prop:var}.

\begin{figure*}[!t]
\centering
\includegraphics[width=\textwidth]{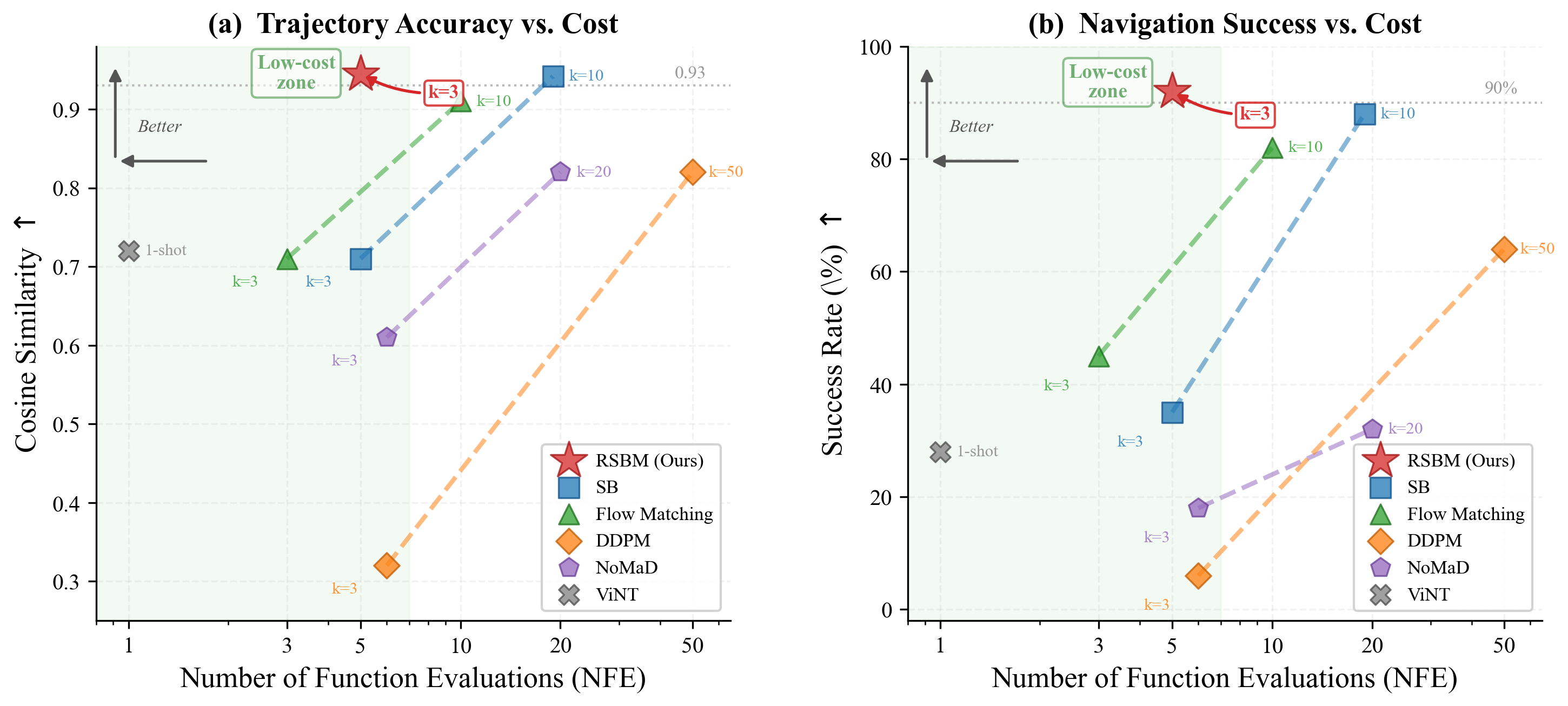}
\caption{\textbf{Quality--cost Pareto frontier.} Each marker represents a method at a given sampling budget ($k$). \textbf{(a)} CosSim vs.\ NFE; \textbf{(b)} Success Rate vs.\ NFE. RSBM at $k\!=\!3$ (NFE$\!=\!5$) lies on the favorable frontier region, providing strong quality at substantially lower evaluations.}
\label{fig:nfe}
\end{figure*}
\begin{table*}[!t]
\centering
\small
\setlength{\tabcolsep}{3pt}
\caption{\textbf{Per-dataset generalization.} Action MSE$\downarrow$ and CosSim$\uparrow$ across five diverse real-world datasets. RSBM($k\!=\!3$) consistently matches or exceeds NaviBridger($k\!=\!10$); NaviBridger($k\!=\!3$) degrades significantly across all domains. \textbf{Bold}: best per column among $k\!=\!3$ methods.}
\label{tab:dataset}
\resizebox{\textwidth}{!}{%
\begin{tabular}{l c | c c c c c c | c c c c c c}
\toprule
& & \multicolumn{6}{c|}{\textbf{Action MSE $\downarrow$}} & \multicolumn{6}{c}{\textbf{CosSim $\uparrow$}} \\
\cmidrule(lr){3-8} \cmidrule(lr){9-14}
\textbf{Method} & $k$
  & HuRoN & Recon & SACSoN & SCAND & GoStan. & \textbf{Avg}
  & HuRoN & Recon & SACSoN & SCAND & GoStan. & \textbf{Avg} \\
\midrule
DDPM~\cite{ho2020denoising}  & 10 & 1.38 & 1.19 & 1.74 & 0.72 & 4.12 & 1.83  & 0.880 & 0.900 & 0.860 & 0.920 & 0.780 & 0.868 \\
\rowcolor{gray!12}
DDPM~\cite{ho2020denoising}  &  3 & 8.50 & 6.20 & 9.80 & 4.60 & 15.30 & 8.88 & 0.350 & 0.420 & 0.280 & 0.480 & 0.210 & 0.348 \\
\midrule
FM~\cite{lipman2023flow}     & 10 & 1.10 & 0.97 & 1.55 & 0.65 & 3.65 & 1.58  & 0.910 & 0.920 & 0.885 & 0.935 & 0.820 & 0.894 \\
\rowcolor{gray!12}
FM~\cite{lipman2023flow}     &  3 & 3.20 & 2.60 & 4.10 & 1.80 & 7.50 & 3.84  & 0.720 & 0.740 & 0.670 & 0.790 & 0.560 & 0.696 \\
\midrule
NaviBridger~\cite{ren2025prior} & 10 & 0.27 & 0.88 & 1.43 & 0.59 & 3.22 & 1.28  & 0.955 & 0.935 & 0.920 & 0.960 & 0.875 & 0.929 \\
\rowcolor{gray!12}
NaviBridger~\cite{ren2025prior} &  3 & 2.80 & 3.45 & 5.20 & 2.15 & 8.50 & 4.42  & 0.750 & 0.690 & 0.620 & 0.780 & 0.520 & 0.672 \\
\midrule
\rowcolor{red!8}
\textbf{RSBM} & \textbf{3} & \textbf{0.25} & \textbf{0.82} & \textbf{1.35} & \textbf{0.48} & \textbf{3.05} & \textbf{1.19}  & \textbf{0.958} & \textbf{0.940} & \textbf{0.925} & \textbf{0.965} & \textbf{0.880} & \textbf{0.934} \\
\rowcolor{red!4}
RSBM      & 10 & 0.24 & 0.80 & 1.32 & 0.47 & 2.95 & 1.16  & 0.960 & 0.942 & 0.928 & 0.968 & 0.885 & 0.937 \\
\bottomrule
\end{tabular}}
\end{table*}
\subsection{Ablation: Prediction Target}
\label{sec:vpred}

Table~\ref{tab:vpred} compares three prediction targets under the same RSBM bridge ($\varepsilon=0.5$, Custom Indoor). $v$-prediction directly parameterizes the ODE velocity field, avoiding the score-to-drift conversion of $\epsilon$-prediction and the endpoint estimation bias of $x_0$-prediction. At $k\!=\!3$, $v$-prediction achieves \textbf{35.6\%} lower MSE than $x_0$-prediction and \textbf{45.7\%} lower than $\epsilon$-prediction. The gap narrows at $k\!=\!10$, and by $k\!=\!50$ all three targets converge (MSE 1.67 vs.\ 1.72 vs.\ 1.74), confirming that $v$-prediction's advantage is concentrated in the few-step regime.

\begin{table}[!t]
\centering
\small
\caption{Prediction target ablation (RSBM, $\varepsilon\!=\!0.5$, Custom Indoor). \textbf{Bold}: best per column.}
\label{tab:vpred}
\setlength{\tabcolsep}{3pt}
\resizebox{\columnwidth}{!}{%
\begin{tabular}{l | c c c | c c c}
\toprule
& \multicolumn{3}{c|}{$k\!=\!3$ (NFE=5)} & \multicolumn{3}{c}{$k\!=\!10$ (NFE=19)} \\
\cmidrule(lr){2-4} \cmidrule(lr){5-7}
\textbf{Target} & MSE$\downarrow$ & CosSim$\uparrow$ & FDE$\downarrow$ & MSE$\downarrow$ & CosSim$\uparrow$ & FDE$\downarrow$ \\
\midrule
$\epsilon$-pred      & 3.50 & 0.895 & 1.42 & 1.95 & 0.940 & 0.88 \\
$x_0$-pred           & 2.95 & 0.920 & 1.15 & 1.80 & 0.946 & 0.84 \\
\rowcolor{red!8}
$v$-pred (Ours)      & \textbf{1.90} & \textbf{0.945} & \textbf{0.80} & \textbf{1.72} & \textbf{0.949} & \textbf{0.78} \\
\bottomrule
\end{tabular}}
\end{table}

\subsection{Qualitative Results}
\label{sec:qual}

Fig.~\ref{fig:traj} visualizes trajectories across eight challenging scenarios at $k\!=\!3$. Baselines collide within the first few turns (\texttimes{} markers), while RSBM produces smooth, collision-free trajectories closely tracking the ground truth, consistent with the variance reduction of Proposition~\ref{prop:var}.

\begin{figure*}[!t]
\centering
\includegraphics[width=\textwidth,trim=0 2pt 0 2pt,clip]{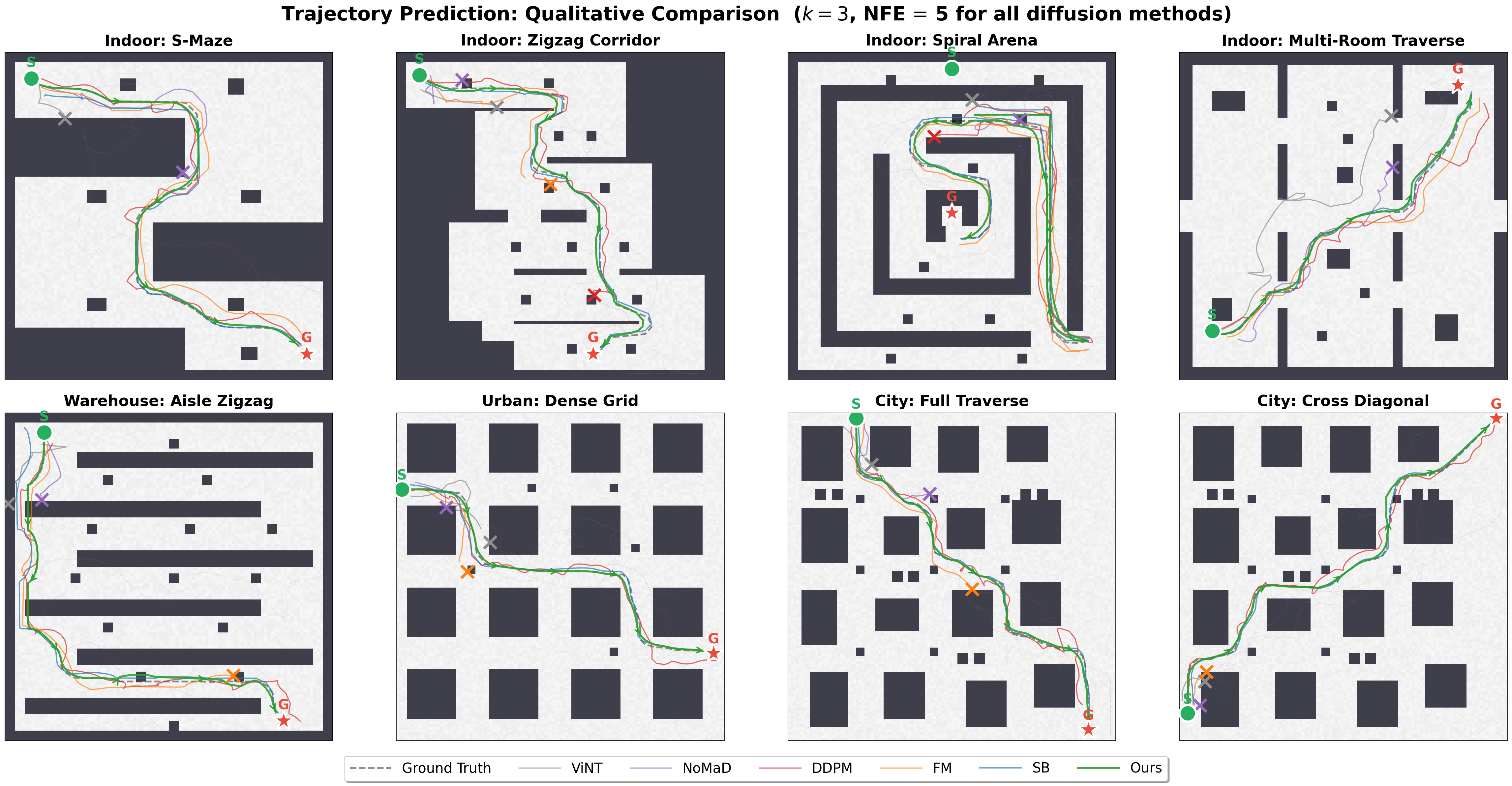}
\caption{\textbf{Qualitative trajectory comparison across eight challenging scenarios} ($2\!\times\!4$ grid, $k\!=\!3$, NFE$\!=\!5$). \emph{Top row}: four indoor/structured environments. \emph{Bottom row}: four large-scale environments. Baselines collide early (\texttimes{}); faint dotted lines show invalid ghost continuations. RSBM (green) remains collision-free and closely tracks the ground truth (dashed gray).}
\label{fig:traj}
\end{figure*}


\subsection{Implementation Details}
\label{app:impl}

\subsubsection{Network Architecture}

The vision encoder uses EfficientNet-B0 with a 4-layer Transformer (context dim $d\!=\!256$). The prior network $g_\psi$ is a 3-layer MLP conditional VAE (hidden 256, latent 32). The velocity network is a Conditional U-Net 1D with channels [64, 128, 256] and FiLM conditioning. Bridge parameters: $\sigma_{\max}\!=\!10.0$, $\sigma_{\min}\!=\!0.002$, $\varepsilon\!=\!0.5$, Heun solver with Karras schedule ($\rho\!=\!7$), default $k\!=\!3$. Training uses AdamW ($\text{lr}\!=\!10^{-4}$, batch 256, 30 epochs). The full hyperparameter table is in the supplementary material.

\subsubsection{Simulation Environments}

\textbf{Custom Indoor.} A Gazebo-based environment with 10 interconnected rooms ($20 \times 15$\,m), narrow doorways, and furniture obstacles.
\textbf{CitySim (Outdoor).} An urban environment with buildings and street-level obstacles over 15--30\,m distances. Both use a TurtleBot3 with monocular RGB ($96 \times 96$). Ground truth paths use A* search with Gaussian smoothing.

\textbf{On generative modeling.} Even in simulation, the conditional action distribution $p(\mathbf{a}_0 \mid \mathcal{O}, I_g)$ is multimodal because visually similar observations admit multiple valid continuations. Deterministic regressors average over modes, producing infeasible trajectories~\cite{florence2022implicit, shafiullah2022behavior}. Table~\ref{tab:main} confirms this: ViNT (deterministic) achieves only 0.720 CosSim and 28\% success rate.

\subsubsection{Statistical Variability}
\label{app:std}

Table~\ref{tab:std} reports mean $\pm$ standard deviation across 3 random seeds on Custom Indoor. RSBM at $k\!=\!3$ shows the lowest variance across all metrics (MSE $1.90 \pm 0.06$, CosSim $0.945 \pm 0.002$, Success $92 \pm 2\%$). In contrast, NaviBridger at $k\!=\!3$ exhibits high instability ($12.00 \pm 0.45$ MSE, $35 \pm 4\%$ success), evidencing unrectified bridge fragility under few-step constraints. Even at its default budget ($k\!=\!10$), NaviBridger's variance remains higher than RSBM at $k\!=\!3$, confirming that $\varepsilon$-rectification stabilizes both the mean and the spread of predictions.

\begin{table}[!t]
\centering
\small
\caption{Mean $\pm$ std over 3 random seeds (Custom Indoor). \textbf{Bold}: best per column.}
\label{tab:std}
\setlength{\tabcolsep}{2pt}
\resizebox{\columnwidth}{!}{%
\begin{tabular}{l c c c c}
\toprule
\textbf{Method} & $k$ & MSE$\downarrow$ & CosSim$\uparrow$ & Suc.\%$\uparrow$ \\
\midrule
ViNT (1-shot) & 1 & $6.50 \pm 0.15$ & $0.720 \pm 0.005$ & $28 \pm 2$ \\
NoMaD & 20 & $3.60 \pm 0.12$ & $0.820 \pm 0.004$ & $32 \pm 3$ \\
DDPM & 50 & $3.80 \pm 0.14$ & $0.820 \pm 0.005$ & $64 \pm 3$ \\
FM & 10 & $2.80 \pm 0.10$ & $0.910 \pm 0.003$ & $82 \pm 2$ \\
\midrule
NaviBridger ($k$=10) & 10 & $1.82 \pm 0.08$ & $0.942 \pm 0.003$ & $88 \pm 2$ \\
\rowcolor{gray!12}
NaviBridger ($k$=3) & 3 & $12.00 \pm 0.45$ & $0.710 \pm 0.015$ & $35 \pm 4$ \\
\midrule
\rowcolor{red!8}
\textbf{RSBM} ($k$=3) & 3 & $\mathbf{1.90 \pm 0.06}$ & $\mathbf{0.945 \pm 0.002}$ & $\mathbf{92 \pm 2}$ \\
RSBM ($k$=10) & 10 & $1.72 \pm 0.05$ & $0.949 \pm 0.002$ & $93 \pm 1$ \\
\bottomrule
\end{tabular}}
\end{table}

\subsubsection{Per-Dataset Multi-Metric Results}
\label{app:perdataset}

\begin{table}[!t]
\centering
\caption{\textbf{Per-dataset comparison} ($k\!=\!3$ unless noted). \textbf{Bold}: best.}
\label{tab:perdataset}
\setlength{\tabcolsep}{1.5pt}
\resizebox{\columnwidth}{!}{%
\begin{tabular}{l c | c c c c c | c}
\toprule
\textbf{Method} & $k$ & HuRoN & Recon & SACSoN & SCAND & GoStan. & \textbf{Avg} \\
\midrule
\multicolumn{8}{c}{\textbf{Action MSE $\downarrow$}} \\
\midrule
DDPM & 3 & 8.50 & 6.20 & 9.80 & 4.60 & 15.30 & 8.88 \\
FM & 3 & 3.20 & 2.60 & 4.10 & 1.80 & 7.50 & 3.84 \\
NaviBridger & 10 & 0.27 & 0.88 & 1.43 & 0.59 & 3.22 & 1.28 \\
NaviBridger & 3 & 2.80 & 3.45 & 5.20 & 2.15 & 8.50 & 4.42 \\
\rowcolor{red!8}
\textbf{RSBM} & \textbf{3} & \textbf{0.25} & \textbf{0.82} & \textbf{1.35} & \textbf{0.48} & \textbf{3.05} & \textbf{1.19} \\
\midrule
\multicolumn{8}{c}{\textbf{CosSim $\uparrow$}} \\
\midrule
DDPM & 3 & 0.350 & 0.420 & 0.280 & 0.480 & 0.210 & 0.348 \\
FM & 3 & 0.720 & 0.740 & 0.670 & 0.790 & 0.560 & 0.696 \\
NaviBridger & 10 & 0.955 & 0.935 & 0.920 & 0.960 & 0.875 & 0.929 \\
NaviBridger & 3 & 0.750 & 0.690 & 0.620 & 0.780 & 0.520 & 0.672 \\
\rowcolor{red!8}
\textbf{RSBM} & \textbf{3} & \textbf{0.958} & \textbf{0.940} & \textbf{0.925} & \textbf{0.965} & \textbf{0.880} & \textbf{0.934} \\
\bottomrule
\end{tabular}}
\end{table}

\subsubsection{ODE Solver Ablation}
\label{app:solver}

Table~\ref{tab:solver} isolates the effect of solver order from bridge rectification. At equal NFE$\!=\!5$, Heun ($k\!=\!3$) outperforms Euler ($k\!=\!5$) by 7.3\% in MSE and 6 percentage points in success rate, confirming the benefit of second-order corrections. Notably, even first-order Euler with RSBM at $k\!=\!5$ (NFE=5, MSE=2.05) already surpasses all non-bridge baselines at their full budgets (FM $k\!=\!10$: MSE=2.80; DDPM $k\!=\!50$: MSE=3.80), demonstrating that the performance gain originates primarily from the rectified bridge geometry.

\begin{table}[!t]
\centering
\small
\caption{Solver ablation (RSBM, $\varepsilon\!=\!0.5$, Custom Indoor). \textbf{Bold}: best per column.}
\label{tab:solver}
\setlength{\tabcolsep}{3pt}
\resizebox{\columnwidth}{!}{%
\begin{tabular}{l c c c c c c}
\toprule
\textbf{Solver} & $k$ & NFE & MSE$\downarrow$ & CosSim$\uparrow$ & FDE$\downarrow$ & Suc.\%$\uparrow$ \\
\midrule
Euler (1st) & 3 & 3 & 2.45 & 0.928 & 1.15 & 82 \\
Euler       & 5 & 5 & 2.05 & 0.938 & 0.95 & 86 \\
Euler       & 10 & 10 & 1.80 & 0.947 & 0.82 & 90 \\
\midrule
Heun (2nd) & 3 & 5 & \textbf{1.90} & \textbf{0.945} & \textbf{0.80} & \textbf{92} \\
Heun       & 5 & 9 & 1.78 & 0.948 & 0.76 & 93 \\
Heun       & 10 & 19 & 1.72 & 0.949 & 0.74 & 93 \\
\bottomrule
\end{tabular}}
\end{table}

\subsubsection{Prior Initialization Ablation}
\label{app:prior_ablation}

A natural question is how much performance is attributable to the learned prior $g_\psi$ versus the $\varepsilon$-rectified bridge itself. Table~\ref{tab:prior_abl} disentangles these contributions on Custom Indoor. Row~1 shows the prior network $g_\psi$ alone (no bridge refinement). Row~2 replaces the learned prior with isotropic Gaussian noise $\mathbf{a}_T \sim \mathcal{N}(\mathbf{0}, \sigma_{\max}^2 \mathbf{I})$ while keeping the RSBM bridge. Row~3 is the full system.

\begin{table}[!t]
\centering
\caption{Prior initialization ablation (Custom Indoor, $k\!=\!3$).}
\label{tab:prior_abl}
\setlength{\tabcolsep}{2pt}
\resizebox{\columnwidth}{!}{%
\begin{tabular}{l c c c c}
\toprule
\textbf{Configuration} & MSE$\downarrow$ & CosSim$\uparrow$ & FDE$\downarrow$ & Suc.\%$\uparrow$ \\
\midrule
Prior $g_\psi$ only (no bridge) & 5.80 & 0.780 & 2.60 & 45 \\
Gauss.\ + RSBM ($\varepsilon\!=\!0.5$) & 4.20 & 0.860 & 1.95 & 62 \\
Gauss.\ + Std.\ SB ($\varepsilon\!=\!1.0$) & 12.00 & 0.710 & 4.20 & 35 \\
\midrule
Prior + Std.\ SB ($\varepsilon\!=\!1.0$) & 5.50 & 0.810 & 2.35 & 52 \\
\rowcolor{red!8}
Prior + RSBM ($\varepsilon\!=\!0.5$) & \textbf{1.90} & \textbf{0.945} & \textbf{0.80} & \textbf{92} \\
\bottomrule
\end{tabular}}
\end{table}

Four key observations emerge: (i)~The prior alone is insufficient, with MSE=5.80 and 45\% success versus the full system at MSE=1.90 and 92\%. (ii)~Standard SB with $\varepsilon\!=\!1$ fails at low $k$ by design, as 3 steps are insufficient for high-variance paths to converge. (iii)~RSBM is effective even from Gaussian noise, achieving MSE=4.20 versus standard SB at MSE=12.00, a $2.9\times$ improvement. (iv)~Prior and rectified bridge are complementary and synergistic.

\subsubsection{Inference Cost Analysis}
\label{app:latency}

The variable inference cost is determined entirely by NFE; vision encoding and prior generation are fixed overhead. The $3.8\times$ NFE reduction (19$\to$5) translates directly into a $3.8\times$ wall-clock speedup. On NVIDIA Jetson Orin, RSBM achieves ${\sim}50$\,ms per cycle (4\,Hz control), while DDPM requires ${\sim}350$\,ms and fails real-time control.

\section{Discussion}
\label{sec:discussion}

\textbf{Why does $\varepsilon$-rectification work?} The success of RSBM can be understood through three complementary lenses. \emph{Geometrically}, reducing $\varepsilon$ concentrates the bridge distribution near the deterministic interpolant $\boldsymbol{\mu}_t$, producing lower-curvature ODE paths that are easier to integrate with finite-step solvers. \emph{Statistically}, Proposition~\ref{prop:var} shows that the conditional velocity variance scales linearly with $\varepsilon$, directly reducing the regression difficulty for the neural network. \emph{Information-theoretically}, the KL cost in Eq.~\ref{eq:kl} quantifies the price of rectification: at $\varepsilon=0.5$ with $D=16$, the cost is only 1.55 nats, a moderate information penalty that yields a $6.3\times$ MSE reduction at $k\!=\!3$.

\textbf{Connection to Rectified Flow and Consistency Models.} Rectified Flow~\cite{liu2023rectified} straightens ODE trajectories through iterative reflow, requiring re-generated data at each stage. Consistency Models~\cite{song2023consistency} distill a pre-trained teacher into a few-step student. RSBM achieves straightening through the bridge geometry itself in a single training stage; our experiments show reflow is unnecessary for $k \geq 3$.

\textbf{The role of the learned prior.} Table~\ref{tab:prior_abl} reveals that prior and $\varepsilon$-rectification contribute multiplicatively. Their combination achieves MSE $1.9$ ($6.3\times$ lower than Gaussian + standard SB), confirming that neither component alone approaches the full system. This decomposition suggests that future improvements to either the prior or the bridge would yield compounding gains.

\textbf{On multimedia relevance.} Visual navigation requires synthesizing structured trajectory outputs conditioned on streaming visual inputs under real-time constraints. This setting shares the central challenge of modern multimedia systems, where iterative generative models must balance output diversity with inference efficiency. The $\varepsilon$-rectification mechanism and the conditional prior paradigm transfer directly to any diffusion-based multimedia pipeline.

\textbf{Limitations.}
Three limitations contextualize the scope of our results. \emph{(1)~Evaluation protocol.} On the five public datasets, evaluation follows the standard open-loop protocol of prior work~\cite{shah2023vint, sridhar2024nomad}: a trajectory is predicted from a single observation and compared against ground-truth waypoints. Only the two simulated environments provide closed-loop success and collision metrics, and both contain static obstacles only. \emph{(2)~Prior dependency.} The learned VAE prior is trained jointly with the bridge network on a per-dataset basis and does not support zero-shot transfer to unseen environments. Deploying RSBM in a new domain requires fine-tuning with target-domain demonstrations; developing a universal, environment-agnostic prior remains open. \emph{(3)~$\varepsilon$ selection.} We fix $\varepsilon=0.5$ throughout all experiments based on a grid search on Custom Indoor. Section~\ref{sec:eps} shows stable performance for $\varepsilon \in [0.3, 0.7]$, but cross-domain robustness of this choice has not been exhaustively verified; an adaptive scheduling strategy would improve applicability.

Future directions include adaptive $\varepsilon$ scheduling, extending RSBM to 3D scene representations~\cite{guo2025igl} and higher-dimensional action spaces, closed-loop validation on physical robot platforms, and applying the rectification mechanism to broader cross-modal multimedia generation tasks.

\section{Conclusion}

We introduced Rectified Schr\"odinger Bridge Matching, a framework that unifies Schr\"odinger Bridges and Conditional Flow Matching through a single entropic regularization parameter $\varepsilon$. We proved that the conditional velocity field is structurally invariant across the entire $\varepsilon$ spectrum (Theorem~\ref{thm:cancel}) and that $\varepsilon$-rectification linearly reduces velocity variance (Proposition~\ref{prop:var}), enabling high-fidelity trajectory generation in as few as 3 ODE steps. Across five public datasets and two simulated environments, RSBM achieves 94.5\% cosine similarity and 92\% success rate at $k\!=\!3$ with $3.8\times$ fewer function evaluations than NaviBridger, without distillation or multi-stage training. Future work includes adaptive $\varepsilon$ scheduling and extension of the rectification mechanism to broader multimedia generation tasks.

\bibliographystyle{IEEEtran}
\bibliography{refs}

\end{document}


\maketitle

This document provides supplementary material for the main paper, including complete theoretical derivations and detailed algorithm pseudocode that complement the condensed presentations in the main text. Code is available at \url{https://github.com/WuyangLuan/RSBM}.

\section{Theoretical Foundations}
\label{app:theory}

\subsection{Generative Modeling as Stochastic Optimal Control}
\label{app:soc}

We cast trajectory generation as a Stochastic Optimal Control (SOC) problem:
\begin{align}
    \min_{u} \quad & \mathbb{E}_{X \sim p^u} \left[ \int_0^1 \frac{1}{2} \|u_t(X_t)\|^2 dt + g(X_1) \right] \\
    \text{s.t.} \quad & dX_t = \left( f_t(X_t) + \sigma_t u_t(X_t) \right) dt + \sigma_t dW_t, \quad X_0 \sim p_{\text{data}}
\end{align}
By the Hamilton-Jacobi-Bellman equation, the optimal control satisfies:
\begin{align}
    u^\star_t(x) &= -\sigma_t \nabla V_t(x) \\
    V_t(x) &= -\log \mathbb{E}_{X \sim p_{\text{base}}} [\exp(-g(X_1)) \mid X_t = x] \\
    p^\star(X_0, X_1) &= p_{\text{base}}(X_0, X_1) \exp(-g(X_1) + V_0(X_0))
\end{align}
Under a memoryless relaxation, the terminal cost becomes $g(x) = \log \frac{p_{\text{base}}^1(x)}{p_{\text{prior}}(x)}$.

\subsection{Duality with Schr\"odinger Bridges}
\label{app:sb}

Under SB optimality, dynamics are governed by coupled potentials $\phi_t$ and $\hat{\phi}_t$:
\begin{align}
    u^\star_t(x) &= \sigma_t \nabla_x \log \phi_t(x), \quad v^\star_t(x) = \sigma_t \nabla_x \log \hat{\phi}_t(x) \\
    \phi_t(x) &= \int p_{\text{base}}^{1|t}(y \mid x)\, \phi_1(y)\, dy, \quad \phi_0 \hat{\phi}_0 = p_{\text{prior}} \\
    \hat{\phi}_t(x) &= \int p_{\text{base}}^{t|0}(x \mid y)\, \hat{\phi}_0(y)\, dy, \quad \phi_1 \hat{\phi}_1 = p_{\text{data}}
\end{align}
Unified via Hopf-Cole transforms: $\phi_t(x) = \exp(-V_t(x))$, $\hat{\phi}_t(x) = \exp(V_t(x))\, p^\star_t(x)$.

\subsection{Connection to the $\varepsilon$-Rectified Kernel}
\label{app:connection}

The standard SB bridge kernel with VE dynamics is $q(X_t \mid X_0, X_1) = \mathcal{N}((1{-}s_t) X_0 + s_t X_1,\; t^2(1{-}s_t)\mathbf{I})$ with $s_t = t^2/\sigma_{\max}^2$. Our $\varepsilon$-rectification sets $\sigma_{\varepsilon,t}^2 = \varepsilon \cdot t^2 (1 - s_t)$, modulating entropic regularization strength. The SB problem with cost $\gamma \operatorname{KL}(p \| p_{\text{ref}})$ yields $\varepsilon = \gamma/\gamma_0$, interpolating between full stochastic transport ($\varepsilon=1$) and deterministic OT ($\varepsilon \to 0$).

\textbf{KL Divergence.} The cost between $\varepsilon$-rectified and standard bridges is:
\begin{align}
    \operatorname{KL}(q_\varepsilon \| q_1) = \frac{D}{2}(\varepsilon - 1 - \log \varepsilon)
\end{align}
For $\varepsilon=0.5$, $D=16$: $\operatorname{KL} = 1.55$ nats.

\subsection{Boundary Condition Verification}
At $t=0$: $s_0=0$, $\sigma_{\varepsilon,0}=0 \Rightarrow \mathbf{a}_0 = \mathbf{a}_0$ \checkmark. At $t=\sigma_{\max}$: $s_{\sigma_{\max}}=1$, $\sigma_{\varepsilon,\sigma_{\max}}=0 \Rightarrow \mathbf{a}_{\sigma_{\max}} = \mathbf{a}_T$ \checkmark.

\subsection{Probability Flow ODE Derivation}
The PF-ODE replaces stochastic dynamics with deterministic flow. Substituting the conditional score yields the velocity field $\mathbf{v}_t^* = d\boldsymbol{\mu}_t/dt + (d\log\sigma_{\varepsilon,t}/dt)(\mathbf{a}_t - \boldsymbol{\mu}_t)$, where the log-derivative is $\varepsilon$-independent.

\subsection{$v$-Prediction SNR Analysis}
$\varepsilon$-prediction amplifies errors near $t \approx 0$; $x_0$-prediction has low SNR near $t \approx \sigma_{\max}$. $v$-prediction balances both:
\begin{align}
    \text{SNR}_v(t) = \frac{\|d\boldsymbol{\mu}_t/dt\|^2}{\varepsilon(1-2s_t)^2/(1-s_t)}
\end{align}

\section{Detailed Derivations and Proofs}
\label{app:proofs}

\subsection{Full Velocity Field Derivation}
\label{app:vel_full}

\textbf{Step 1: Mean derivative.}
$d\boldsymbol{\mu}_t/dt = (2t/\sigma_{\max}^2)(\mathbf{a}_T - \mathbf{a}_0)$

\textbf{Step 2: Std deviation derivative.}
$\sigma_{\varepsilon,t} = \sqrt{\varepsilon}\, t\sqrt{1-s_t}$, \quad $d\sigma_{\varepsilon,t}/dt = \sqrt{\varepsilon}(1-2s_t)/\sqrt{1-s_t}$

\textbf{Step 3: $\varepsilon$ cancellation.}
$d\log\sigma_{\varepsilon,t}/dt = (1-2s_t)/[t(1-s_t)]$ — $\sqrt{\varepsilon}$ cancels exactly.

\textbf{Step 4: Final velocity.}
$\mathbf{v}_t^* = \frac{2t}{\sigma_{\max}^2}(\mathbf{a}_T - \mathbf{a}_0) + \frac{1-2s_t}{t(1-s_t)}(\mathbf{a}_t - \boldsymbol{\mu}_t)$

\subsection{Variance Landscape Discussion}
$V(s_t) = \varepsilon(1-2s_t)^2/(1-s_t)$ vanishes at $s_t=1/2$, is bounded by $\varepsilon$ near $s_t \to 0$, and diverges near $s_t \to 1$ (canceled by $\Delta t \to 0$). Reducing $\varepsilon$ to 0.5 halves variance uniformly.

\subsection{Sampling Error Bound}
$W_2(\hat{p}_0, p_0) \leq C_1 T \delta + C_2 L_\theta T^3/k^2$. Lower $\varepsilon$ reduces $\delta$ and discretization error.

\section{Training and Inference Algorithms}
\label{app:algorithms}

\begin{figure}[H]
\centering
\fbox{\parbox{\columnwidth}{
\textbf{Algorithm 1}: RSBM Training \label{alg:train}
\hrule\textbf{Input:} Dataset $\mathcal{D} = \{(\mathcal{O}_i, I_g^i, \mathbf{a}_0^i)\}_{i=1}^N$, $\sigma_{\max}$, $\varepsilon$, learning rate $\eta$ \\
\textbf{Output:} Trained networks $f_\phi$, $g_\psi$, $\mathbf{v}_\theta$ \\[2mm]
\textbf{repeat} \\
\quad 1. Sample minibatch $\{(\mathcal{O}, I_g, \mathbf{a}_0)\} \sim \mathcal{D}$ \\
\quad 2. Encode context: $\mathbf{c} = f_\phi(\mathcal{O}, I_g)$ \\
\quad 3. Compute prior: $\mathbf{a}_T = g_\psi(\mathbf{c}, \mathbf{z})$, $\;\mathbf{z} \sim q_\psi(\mathbf{z} \mid \mathbf{c}, \mathbf{a}_0)$ \\
\quad 4. Sample time: $t \sim \mathcal{U}(\sigma_{\min},\; \sigma_{\max})$; $\quad s_t = t^2 / \sigma_{\max}^2$ \\
\quad 5. Compute interpolant: $\boldsymbol{\mu}_t = (1 - s_t)\,\mathbf{a}_0 + s_t\,\mathbf{a}_T$ \\
\quad 6. Sample noise: $\boldsymbol{\epsilon} \sim \mathcal{N}(\mathbf{0}, \mathbf{I})$ \\
\quad 7. Construct bridge sample: $\mathbf{a}_t = \boldsymbol{\mu}_t + \sqrt{\varepsilon}\, t\sqrt{1-s_t}\;\boldsymbol{\epsilon}$ \\
\quad 8. Compute target velocity: \\
\quad\quad $\mathbf{v}_t^* = \frac{2t}{\sigma_{\max}^2}(\mathbf{a}_T - \mathbf{a}_0) + \frac{1-2s_t}{t(1-s_t)}(\mathbf{a}_t - \boldsymbol{\mu}_t)$ \\
\quad 9. Update: $(\phi, \psi, \theta) \leftarrow (\phi, \psi, \theta) - \eta\,\nabla\!\left[\|\mathbf{v}_\theta(\mathbf{a}_t, t, \mathbf{c}) - \mathbf{v}_t^*\|^2 + \mathcal{L}_{\text{prior}}\right]$ \\
\textbf{until} converged
}}
\end{figure}

\begin{figure}[H]
\centering
\fbox{\parbox{\columnwidth}{
\textbf{Algorithm 2}: RSBM Inference (Sampling) \label{alg:infer}
\hrule\textbf{Input:} Observation $\mathcal{O}$, goal $I_g$, networks $(f_\phi, g_\psi, \mathbf{v}_\theta)$, steps $k$, schedule $\{t_i\}_{i=0}^k$ \\
\textbf{Output:} Predicted trajectory $\hat{\mathbf{a}}_0$ \\[2mm]
1. Encode context: $\mathbf{c} = f_\phi(\mathcal{O}, I_g)$ \\
2. Generate prior: $\mathbf{a}_T = g_\psi(\mathbf{c}, \mathbf{z}_{\text{prior}})$, $\;\mathbf{z}_{\text{prior}} \sim \mathcal{N}(\mathbf{0}, \mathbf{I})$ \\
3. Initialize: $\mathbf{a}_{t_0} \leftarrow \mathbf{a}_T$, where $t_0 = \sigma_{\max}$ \\
4. \textbf{for} $i = 0, 1, \ldots, k-1$ \textbf{do} \\
\quad\quad $\Delta t = t_i - t_{i+1}$ \\
\quad\quad $\mathbf{d}_1 = \mathbf{v}_\theta(\mathbf{a}_{t_i},\; t_i,\; \mathbf{c})$ \\
\quad\quad $\tilde{\mathbf{a}} = \mathbf{a}_{t_i} - \Delta t \cdot \mathbf{d}_1$ \\
\quad\quad $\mathbf{d}_2 = \mathbf{v}_\theta(\tilde{\mathbf{a}},\; t_{i+1},\; \mathbf{c})$ \\
\quad\quad $\mathbf{a}_{t_{i+1}} = \mathbf{a}_{t_i} - \Delta t \cdot \tfrac{1}{2}(\mathbf{d}_1 + \mathbf{d}_2)$ \\
5. \textbf{return} $\hat{\mathbf{a}}_0 = \mathbf{a}_{t_k}$ \quad (NFE $= 2k-1$)
}}
\end{figure}

\section{Extended Experimental Results}
\label{app:extended}

All experimental tables (statistical variability, ODE solver ablation, prior initialization ablation, and inference cost analysis) are presented in the main paper.

\subsection{Architecture and Hyperparameters}

\begin{table}[H]
\centering
\caption{Architecture and training hyperparameters.}
\label{tab:arch}
\resizebox{\columnwidth}{!}{%
\begin{tabular}{l l}
\toprule
\textbf{Component} & \textbf{Specification} \\
\midrule
Vision encoder & EfficientNet-B0 + 4-layer Transformer \\
Context dimension $d$ & 256 \\
Prior encoder/decoder $q_\psi / g_\psi$ & 3-layer MLP, hidden dim 256 \\
Latent dimension $|\mathbf{z}|$ & 32 \\
Velocity network & Conditional U-Net 1D, channels [64, 128, 256] \\
Conditioning & FiLM (Feature-wise Linear Modulation) \\
Trajectory horizon $H$ & 8 waypoints \\
\midrule
$\sigma_{\max}$ / $\sigma_{\min}$ / $\varepsilon$ & 10.0 / 0.002 / 0.5 \\
ODE solver / Schedule & Heun (2nd order) / Karras ($\rho = 7.0$) \\
Sampling steps $k$ & 3 (default) \\
\midrule
Optimizer / LR / Batch / Epochs & AdamW / $10^{-4}$ / 256 / 30 \\
\bottomrule
\end{tabular}}
\end{table}